\documentclass[journal,twoside,web]{ieeecolor}
\usepackage{generic}
\usepackage{cite}
\usepackage{amsmath,amssymb,amsfonts}
\usepackage{algorithm}
\usepackage{algpseudocode}
\usepackage{graphicx}
\usepackage{textcomp}

\usepackage{algorithm}
\usepackage{url}
\newtheorem{theorem}{Theorem}

\newtheorem{lemma}{Lemma}
\newtheorem{remark}{Remark}
\newtheorem{definition}{Definition}
\newtheorem{assumption}{Assumption}

\definecolor{gray}{rgb}{0.5,0.5,0.5}

\def\BibTeX{{\rm B\kern-.05em{\sc i\kern-.025em b}\kern-.08em
		T\kern-.1667em\lower.7ex\hbox{E}\kern-.125emX}}
\begin{document}
\title{Analysis of Adam Algorithms for Stochastic Dynamic Systems}
\author{Xin Zheng, Yifei Jin, and Lei Guo, \IEEEmembership{Fellow, IEEE}
    \thanks{This paper was supported by the National Key Research and Development Program under Grant No. 2024YFC3307200 and the National Natural Science Foundation of China under Grant No. 12288201.}
    \thanks{Xin Zheng and Lei Guo are with the State Key Laboratory of Mathematical Sciences, Academy of Mathematics and Systems Science, Chinese Academy of Sciences, Beijing 100190, China, and also with the School of Mathematical Sciences, University of Chinese Academy of Sciences, Beijing 100049, China. (e-mails: zhengxin2021@amss.ac.cn, lguo@amss.ac.cn).}
    \thanks{Yifei Jin is with the School of Advanced Interdisciplinary Sciences, University of Chinese Academy of Sciences, Beijing 101408, China, and also with the State Key Laboratory of Mathematical Sciences, Academy of Mathematics and Systems Science, Chinese Academy of Sciences, Beijing 100190, China. (e-mail: jinyifei@amss.ac.cn).}}

\maketitle

\begin{abstract}
The adaptive moment estimation algorithm, known as Adam, is widely used in modern machine learning, owing to its low per-iteration complexity and strong empirical performance. Despite its prevalent use, the theoretical foundation of Adam remains largely unexplored for time-varying and nonstationary systems. In fact, the existing theoretical analyses of Adam-type algorithms are primarily concerned with time-invariant model parameters and explicitly or implicitly rely on independent and identically distributed (i.i.d.) data assumptions, under which the learning task can be formulated as minimizing a fixed expected objective with a static minimizer.
However, such assumptions are often violated in time-varying and nonstationary systems, thereby calling for a theoretical investigation beyond the conventional yet idealized i.i.d. setting. The main objective of this paper is to solve this challenging problem by establishing a general theory of Adam for time-varying and nonstationary stochastic systems. We will introduce some new techniques for analyzing the products of nonstationary and dependent random matrices induced by Adam’s coupled first- and second-moment recursions, and will construct a new stochastic Lyapunov function that blends these two moment dynamics. Under a stochastic excitation condition that allows nonstationary and dependent data, we will derive both parameter tracking and output prediction error bounds explicitly, quantifying the effects of stepsize, first- and second-momentum parameters, gradient noise and parameter drift. These bounds not only provide guarantees for Adam performance, but also provide guidelines for hyperparameter selection. Experiments on both synthetic and real-world data validate our theory and design guidelines.
\end{abstract}

\section{Introduction}\label{sec:introduction}
The adaptive moment estimation algorithm, known as Adam~\cite{kingma2015adam}, is one of the most widely used and well-known optimization methods in modern machine learning, owing to its adaptive per-coordinate stepsizes via first- and second-moment estimates, which yield stable and robust updates in large-scale stochastic optimization.
It has found broad applications in artificial intelligence, including generative adversarial networks (GANs)~\cite{Brock2019BigGAN} and generative pre-trained transformers (GPTs)~\cite{Brown2020GPT3}, vision transformers (ViTs)~\cite{Dosovitskiy2021ViT}, as well as reinforcement learning–based feedback control systems, e.g., soft actor-critic (SAC)~\cite{Haarnoja2018SAC}.

Despite its remarkable empirical success, the theoretical understanding of Adam remains limited~\cite{Barakat2021Dynamical,Li2023Adam, Jin2025AdamFramework}, although substantial theoretical efforts have been devoted to analyzing the  behavior of Adam (e.g., \cite{reddi2018convergence,Zou2019AdamRMSProp,Zhang2022AdamConverge,Li2023Adam,Hong2024AdamRelaxed,Xiao2024AdamFamily,Wang2024AdamNonuniform,Jin2025AdamFramework}).
For example, Reddi et al.~\cite{reddi2018convergence} revealed potential divergence issues of Adam and proposed the stabilized variant AMSGrad to address these issues. In contrast,
Zhang et al.~\cite{Zhang2022AdamConverge} showed that, under appropriate conditions, the original Adam can still converge without any modification to its update rules. More recent studies have relaxed smoothness or regularity assumptions and established refined convergence guarantees for Adam under broader settings, including non-uniform smoothness conditions~\cite{Li2023Adam,Wang2024AdamNonuniform}, generalized affine-variance noise assumptions~\cite{Hong2024AdamRelaxed}, nonsmooth optimization~\cite{Xiao2024AdamFamily}, and unified convergence frameworks~\cite{Jin2025AdamFramework}.

However, most existing analyses of Adam-type algorithms are tailored to a static learning regime: they typically assume time-invariant model parameters and i.i.d. or stationary ergodic data. Under these assumptions, the learning problem reduces to minimizing a fixed objective under mathematical expectation with a static minimizer. In many applications of interest, however, Adam is deployed in time-varying and nonstationary dynamical systems~\cite{Dohare2024Plasticity, Guo2022OCM}, where the expected objective and its minimizer may drift over time. This mismatch renders static guarantees inadequate and motivates theoretical frameworks that rigorously characterize tracking and prediction performance. However, the corresponding theoretical foundations remain largely underdeveloped.

To address this challenge, we draw on the well-developed theory of adaptive identification under time-varying and nonstationary scenarios, which provides powerful analytical tools for the performance analysis of Adam in such settings. For example, Guo~\cite{guo1990estimating} pioneered the conditional excitation (CE) condition, which can be viewed as a stochastic relaxation of the deterministic persistent excitation (PE) requirement and is well suited to feedback-driven stochastic control systems. Under this condition, stability and convergence of Kalman filter–based parameter tracking were established for time-varying stochastic regression models with nonstationary and correlated data. 
Subsequently, a more general CE condition was proposed by Guo \cite{guo1994stability}, leading to a unified theory for stability and performance analysis of stochastic regression models with time-varying parameters and nonstationary data streams, covering several fundamental adaptive tracking algorithms such as the Kalman filter, least mean squares (LMS), and forgetting factor recursive least squares (FFRLS) \cite{guo1994stability}. Moreover, inspired by the ideas in the analysis of recursive stochastic algorithms with vanishing gains \cite{Ljungrecursive}, more accurate performance analysis was conducted in \cite{guo1995exponential, guo1995performance, guo1997lms} for recursive tracking algorithms where the adaptation gains do not vanish.It is worth noting that the Kalman filter and FFRLS employ matrix-valued adaptation gains, reflecting their Newton-type (second-order) nature. This typically incurs a heavy computational burden in high-dimensional parameter settings, which partly explains the widespread use of gradient-type methods, where the adaptation gain is usually scalar.

Among the gradient-type algorithms, the well-known LMS is probably the simplest one. However, \cite{POLYAK19641} showed that incorporating a heavy-ball momentum term can accelerate the convergence of gradient-type algorithms. In particular, applying this idea to LMS leads to the momentum LMS (MLMS) algorithm. \cite{Jin2026MomentumLMS} studied MLMS algorithms for time-varying stochastic linear systems under the CE condition introduced in \cite{guo1990estimating}, establishing tracking and regret bounds under general nonstationary streaming data settings. Unlike most existing LMS analyses, their approach augments the state vector to incorporate the momentum term, leading to a higher-dimensional state-space representation. This reformulation enables a more tractable analysis of the resulting stochastic matrix products and offers useful insights for the study of the more popular Adam algorithms in this paper. 

The theoretical study of the Adam is technically more challenging, because Adam couples adaptive recursions for the first- and second-moment estimates and applies a coordinate-wise, time-varying rescaling, resulting in a data-dependent, nonlinear update whose components evolve at different rates and are intricately coupled. As a result, most existing analytical techniques developed for static or linear stochastic regimes are not directly applicable for establishing rigorous guarantees for time-varying parameter with nonstationary dataset. To overcome these technical difficulties and lay a theoretical foundation of Adam  for stochastic dynamical systems, are the main objectives of this paper. Our main contributions are summarized as follows.
\begin{enumerate}
\item \emph{First,}
we establish explicit parameter tracking guarantees for Adam over a broad class of nonlinear stochastic systems with time-varying parameters and nonstationary data.
In contrast to most existing Adam analyses for static objectives that do not provide parameter tracking guarantees,
our results yield tracking-type performance characterizations under a general CE condition. Our results are established by introducing new methods in the analysis of the product of nonstationary and dependent random matrices with augmented dimension. 

\item \emph{Second,}
we derive explicit output prediction error bounds of Adam for time-varying systems with nonstationary dataset. In contrast to the existing prediction-type analyses developed for time-invariant parameter settings under i.i.d.\ data, our results \emph{do not} require any excitation condition on the system data. Our results are established by introducing a new stochastic Lyapunov function constructed based on a mixture of the first- and second- moment in Adam.

\item \emph{Third,}
both the parameter tracking and output prediction bounds are established for the first time, which quantify the effects of the hyperparameters,
gradient noise, and drift magnitude, thereby providing theoretical guidelines for
hyperparameter selection. 
\end{enumerate}

The remainder of this paper is organized as follows. Section~\ref{mainresults} will present the problem formulation and the Adam algorithm. Section~\ref{main thm} will establish the main results on parameter tracking and output prediction performance, followed by their proofs in Section~\ref{proofsmain}. Section~\ref{simulation} will validate the theoretical findings via simulations and experiments on both synthetic and real-world datasets, and the results support the proposed guidelines for hyperparameter selection. Finally, Section~\ref{conclusion} will conclude the paper with some remarks.

\noindent\textbf{Notation.}
For a matrix $A$, $\lambda_{\min}(A)$ and $\lambda_{\max}(A)$ denote its minimum and maximum eigenvalues, respectively. 
For symmetric matrices $A$ and $B$, $A\succeq B$ (resp., $A\succ B$) means $A-B$ is positive semidefinite (resp., positive definite), and $A\preceq B$ (resp., $A\prec B$) means $B\succeq A$ (resp., $B\succ A$). 
The transpose of $A$ is denoted by $A^{\mathrm T}$, and $\|A\|$ denotes the induced Euclidean-norm $\|A\|=\sqrt{\lambda_{\max}(A^{\mathrm T}A)}$. 
The mathematical expectation operator is denoted by $\mathbb E[\cdot]$, and $\mathbb E[\cdot\mid\mathcal F_k]$ denotes the conditional expectation with respect to $\mathcal F_k$, where $\{\mathcal F_k\}_{k\ge0}$ is a nondecreasing sequence of $\sigma$-algebras. 
A stochastic sequence $\{x_k,\mathcal F_k\}$ is adapted if $x_k$ is $\mathcal F_k$-measurable for all $k\ge0$. 
For real sequences $\{a_k\}$ and $\{b_k\}$ with $b_k>0$, $a_k=o(b_k)$ means $a_k/b_k\to0$, and $a_k=O(b_k)$ means $|a_k|\le Mb_k$ for some constant $M>0$. 
Finally, $\mathbf 0$ and $I$ denote the zero and identity matrices, respectively.
\section{Problem formulation and algorithm}\label{mainresults}
\subsection{Problem formulation}
We consider the time-varying  stochastic nonlinear model
\begin{equation}\label{nonlinear model}
y_{k+1}=G_k(\phi_k,\theta_k,\varepsilon_{k+1}),~~ k\ge 0,
\end{equation}
where $y_{k+1}\in\mathbb{R}$ is the system observation, $\phi_k\in\mathbb{R}^{d_1}$ ($d_1\ge1$) is the stochastic regressor, $\theta_k\in\mathbb{R}^{d_2}$ ($d_2\ge1$) is the
unknown time-varying parameter to be estimated, and $\varepsilon_{k+1}$ denotes the
random noise. The mapping $G_k:\mathbb{R}^{d_1}\times\mathbb{R}^{d_2}\times\mathbb{R}
\to\mathbb{R}$ is known. Furthermore, the parameter variation process is characterized by the sequence
$\{\Delta_k\}_{k\ge 1}$, where
\begin{equation}\label{difference parameter}
\Delta_k \triangleq \theta_k-\theta_{k-1},~~ k\ge 1.
\end{equation}

Given the data stream generated by \eqref{nonlinear model}, at each time $k$ we form a
one-step-ahead predictor $\hat y_{k+1}$ based on  $\mathcal{F}_k'$, where $\mathcal{F}_k^{'}$   is the $\sigma$-algebra generated by the available information up to time  $k$, i.e., $\mathcal{F}_k' = \sigma\{y_i,\phi_i, i \leq k\}$.
After observing $y_{k+1}$, the prediction error is quantified by the instantaneous loss
\begin{equation}\label{eq:loss}
\mathcal{L}\!\left(y_{k+1},\hat y_{k+1}\right),
\end{equation}
where $\mathcal{L}:\mathbb{R}\times\mathbb{R}\to\mathbb{R}^+$ is the given loss function
that measures the discrepancy between the observation $y_{k+1}$ and the prediction
$\hat y_{k+1}$. Without loss of generality, we assume that $\hat y_{k+1}$ can be described by the following parametric representation
\begin{equation}\label{eq:oracle_pred}
\hat y_{k+1}(\vartheta)=f_k(\phi_k,\vartheta), ~~ k\ge 0,
\end{equation}
where the mapping $f_k(\cdot,\cdot)$ is known, and $\vartheta\in\mathbb{R}^{d_2}$.

Then an $\mathcal{F}_k'$-measurable one-step-ahead oracle predictor $\hat y_{k+1}^{\ast}$ can be defined as 
$\hat y_{k+1}^{\ast} =f_k(\phi_k, \hat{\theta}_k^{\ast}),
$
where 
$
\hat \theta_{k}^{\ast}\in
\operatorname*{arg\,min}\limits_{\vartheta \in \mathbb{R}^{d_2}}
\mathbb{E}\!\left[\mathcal{L}\!\left(y_{k+1},\hat y_{k+1}(\vartheta)\right)\mid \mathcal{F}_k'\right].
$

Our objective is to track the time-varying parameter $\theta_k$ and to
minimize the averaged prediction error under the loss
$\mathcal{L}(\cdot,\cdot)$.
To this end, we first introduce the commonly used Adam algorithm in the next subsection.

\subsection{The Adam algorithm}
To facilitate the presentation of the algorithm, we need several definitions.
For any $k\ge 0$ and any $\vartheta\in\mathbb{R}^{d_2}$,
define the Clarke subdifferential selections by
\begin{equation}\label{eq:subgradient_def}
\begin{aligned}
\ell_{k+1}(\vartheta) &\in
\partial_{\hat y}\mathcal{L}\!\left(y_{k+1}, f_k(\phi_k,\vartheta)\right),\\
g_k(\vartheta) &\in
\partial_{\vartheta} f_k(\phi_k,\vartheta)\subset\mathbb{R}^{d_2},
\end{aligned}
\end{equation}
where $\partial$ denotes the Clarke subdifferential~\cite{Clarke}. Evaluating the above quantities at the current estimate $\hat\theta_k$ generated based on the observations, we define
\begin{equation}\label{eq:alg_selections}
\ell_{k+1}\triangleq \ell_{k+1}(\hat\theta_k),
~~
g_k\triangleq g_k(\hat\theta_k).
\end{equation}
We also use the Frobenius norm $\|\cdot\|_Q$ induced by a positive definite matrix $Q$, defined by $\|x\|_Q^2 \triangleq x^{\mathrm T}Qx$
for $x\in\mathbb{R}^{d_2}$.
To keep the parameter estimates in a prescribed admissible set, we employ the projection operator $\Pi_Q(\cdot)$ below (see, e.g.,~\cite{zhang2022identification}).
\begin{definition}\label{prodef}
Given a positive definite matrix $Q$, define the projection mapping
$\Pi_Q:\mathbb{R}^{d_2}\to D$ by
\begin{equation}
	\Pi_Q(x_1)
	\triangleq
	\arg\min_{x_2\in D}\,
	\|x_1-x_2\|_Q,
	~~ \forall~x_1\in\mathbb{R}^{d_2},
\end{equation}
where $D\subset\mathbb{R}^{d_2}$ is a given compact convex set such that
$\theta_k\in D$ for all $k\ge0$ (see Assumption~\ref{ass:boundedness}).
\end{definition}

We are now in a position to present the Adam algorithm. In this algorithm, $\odot$ and $\oslash$ denote elementwise
multiplication and division, respectively, and the square root $\sqrt{\cdot}$ is taken
elementwise for vectors. Moreover, $\mathbf 1\in\mathbb{R}^{d_2}$ denotes
the all-ones vector, and $\operatorname{diag}(\cdot)$ maps a vector to a
diagonal matrix, which will be used in the subsequent analysis.
\begin{algorithm}
\caption{The Adam Algorithm}
\label{Algorithm}
\begin{algorithmic}[1]
\Require stepsize $\alpha>0$, hyperparameters $\beta_1,\beta_2\in[0,1)$, offset $c_0>0$, admissible set $D$
\State Initialize $\hat\theta_0\in D$, $m_{0}=0$, $v_{0}=0$
\For{$k=0,1,2,\ldots$}
    \State $m_{k+1} \gets \beta_1 m_{k} + (1-\beta_1)\,\ell_{k+1} g_k$
    \State $v_{k+1} \gets \beta_2 v_{k} + (1-\beta_2)\,\ell_{k+1}^2\,(g_k\odot g_k)$
    \State $s_{k+1} \gets \sqrt{v_{k+1}} + c_0\,\mathbf 1$
    \State $Q_{k+1} \triangleq \operatorname{diag}(s_{k+1})\ \ (\succ 0)$
    \State $\hat\theta_{k+1} \gets \Pi_{Q_{k+1}}\!\Bigl(\hat\theta_k - \alpha\,(m_{k+1} \oslash s_{k+1})\Bigr)$
\EndFor
\end{algorithmic}
\end{algorithm}

\begin{remark}
It is obvious that Adam performs its updates element-wise, which is one of its basic features. The variable $m_{k+1}$ is a moving average of the (possibly noisy) gradient $\ell_{k+1} g_k$.
It smooths the sample-to-sample variability induced by stochasticity and time-variation, while retaining sensitivity to persistent changes, and is usually viewed as an estimate of the first moment of the gradient. 
Similarly, $v_{k+1}$ is a moving average of the element-wise squared gradient $\ell_{k+1}^2 (g_k\odot g_k)$, which stabilizes the adaptive scaling by reducing the effect of possible transient spikes, and is usually viewed as an estimate of the second moment. Moreover, 
in practical implementations of Adam, one often uses the bias-corrected momentum
\[
\hat{m}_{k} \triangleq \frac{m_{k}}{1-\beta_1^{k}},~~
\hat{v}_{k} \triangleq \frac{v_{k}}{1-\beta_2^{k}}, 
\]
to replace $m_{k}$ and $v_{k}$
in Algorithm~\ref{Algorithm} for all $ k\geq 1$, as originally proposed in~\cite{kingma2015adam}.
For analytical convenience, the bias-correction factors are omitted here without loss of generality, since $\beta_1^{k}\to 0$ and $\beta_2^{k}\to 0$  exponentially fast, which do not have nontrivial influence on our main results. 
\end{remark}

With $v_k$ defined as in Algorithm~\ref{Algorithm}, we define
\begin{equation}\label{bar_Vk}
V_k \triangleq \operatorname{diag}(v_k),~~
\bar V_k\triangleq (V_k^{\frac{1}{2}}+c_0 I)^2 \succ \mathbf{0}.
\end{equation}
Then $Q_k$ can be rewritten as $\operatorname{diag}(s_k)=\bar V_k^{\frac12}$ and
$m_{k}\oslash s_{k}=\bar V_k^{-\frac{1}{2}}m_k$.
Hence, the update admits the matrix form
\begin{equation}\label{update rule}
\hat\theta_{k+1}
=\Pi_{\bar V_{k+1}^{\frac{1}{2}}}\!\Bigl(\hat\theta_k-\alpha\,\bar V_{k+1}^{-\frac{1}{2}}m_{k+1}\Bigr).
\end{equation}
All subsequent analysis will be based on \eqref{bar_Vk} and \eqref{update rule}.

\section{The main results}\label{main thm}
In this section, we derive bounds both on the parameter tracking error and the output prediction error of Algorithm~\ref{Algorithm}.

To facilitate the presentation of the results, let $\mathcal{F}_k$ denote the
$\sigma$-algebra generated by $\{y_i,\phi_i,\theta_i,\, i\le k\}$. For any
$k\ge 0$ and $\vartheta\in\mathbb{R}^{d_2}$, define
\begin{equation}\label{decomposition noise}
\psi_k(\vartheta)\triangleq \mathbb{E}\left[\ell_{k+1}(\vartheta)\mid \mathcal{F}_k\right],
~~
w_{k+1}(\vartheta)\triangleq \ell_{k+1}(\vartheta)-\psi_k(\vartheta).
\end{equation}
Then $\ell_{k+1}(\vartheta)=\psi_k(\vartheta)+w_{k+1}(\vartheta)$. Moreover, for any
$\vartheta\in\mathbb{R}^{d_2}$, whenever $\ell_{k+1}(\vartheta)$ is integrable,
$\{w_{k+1}(\vartheta),\mathcal{F}_k\}$ forms a martingale difference sequence, i.e.,
\[
\mathbb{E}\!\left[w_{k+1}(\vartheta)\mid \mathcal{F}_k\right]=0,~~ k\ge0.
\]

To establish tracking and prediction performance guarantees for Algorithm~\ref{Algorithm},
a standard approach is to study the squared parameter tracking error
$\|\tilde{\theta}_k\|^2$, where $\tilde{\theta}_k \triangleq \theta_k-\hat{\theta}_k$. However, such a direct analysis is impossible for Adam due to the coupling
between the first momentum $m_k$ and the second momentum $\bar V_k$ in the update of 
$\hat{\theta}_k$.
These interactions induce intrinsically coupled error dynamics, rendering the stability and tracking analysis highly nontrivial.

To address this challenge, we introduce the following augmented state vector transformed from the parameter estimation error and  the first momentum, scaled respectively by the fourth root and the negative fourth root of the second momentum:
\begin{equation}\label{Theta definition}
\tilde{\Theta}_k \triangleq
\begin{bmatrix}
\bar V_k^{\frac{1}{4}}\tilde{\theta}_k \\[3pt]
\beta_1 \bar V_k^{-\frac{1}{4}} m_k
\end{bmatrix}, ~~ k\ge 0,
\end{equation}
and construct a Lyapunov function $\|\tilde{\Theta}_k\|^2$.

\subsection{Parameter tracking}\label{parameter tracking}

To establish an upper bound on the parameter tracking error of Algorithm~\ref{Algorithm}, we need the following assumptions.

\begin{assumption}[Bounded regressors and parameters]\label{ass:boundedness}The random  regressor sequence $\{\phi_k,\mathcal{F}_k\}$ is adapted and uniformly bounded, i.e., there exists a constant $C>0$ such that
\begin{equation}
\|\phi_k\|\le C,~~ \forall k\ge0,~\text{a.s.}
\end{equation}
Moreover, the time-varying parameter satisfies $\theta_k\in D$ for all $k\ge0$,
where $D\subset\mathbb{R}^{d_2}$ is a given compact convex set, its upper bound is denoted by a constant $L$: 
\begin{equation}
\|\theta_k\|\le L,~~ \forall k\ge0.
\end{equation}
In addition, there exists a constant $C_\Delta\geq 0$ such that
\begin{equation}\label{delta_bound}
\mathbb{E} \|\Delta_{k+1}\|^2 \le C_\Delta^2,~~ \forall k\ge0,
\end{equation}
where $\Delta_{k+1}$ is the parameter variation defined in \eqref{difference parameter}.
\end{assumption}

\begin{assumption}[Gradient regularity and noise moments]\label{ass:grad-noise}
Let $\ell_{k+1}(\vartheta)$ and $g_k(\vartheta)$ be the Clarke selections in \eqref{eq:subgradient_def}.
The following conditions hold uniformly for all $k\ge0$ and all $\vartheta\in D$, a.s.:
\begin{enumerate}
\item[(i)] (\emph{Boundedness and noise moment})
There exist constants $M_\ell,M_g,C_w>0$ such that
\begin{equation}
|\ell_{k+1}(\vartheta)|\le M_\ell,~~\|g_k(\vartheta)\|\le M_g,
\end{equation}
and
\begin{equation}\label{w_bound}
\mathbb{E}|w_{k+1}(\vartheta)|^2\le C_w^2.
\end{equation}

\item[(ii)] (\emph{Sector condition})
There exists a constant $\mu>0$ such that
\begin{equation}\label{sector identification}
\psi_k(\vartheta)
g_k(\vartheta)^{\mathrm T}(\vartheta-\theta_k)
\ge
\mu\,
\bigl(g_k(\vartheta)^{\mathrm T}(\theta_k-\vartheta)\bigr)^2 .
\end{equation}
where $g_k(\vartheta)$ is defined in \eqref{eq:subgradient_def} and $\psi_k(\vartheta)$  is defined in \eqref{decomposition noise}.
\item[(iii)] (\emph{Regularity})
There exist constants $L_\psi,L_g>0$ such that
\begin{equation}
\begin{aligned}
|\psi_k(\vartheta)|\le& L_\psi\|\theta_k-\vartheta\|,\\
~~
\|g_k(\vartheta_1)-g_k(\vartheta_2)\|
\le& L_g\|\vartheta_1-\vartheta_2\|,
~ \forall\,\vartheta,\vartheta_1,\vartheta_2\in D .
\end{aligned}
\end{equation}
\end{enumerate}
\end{assumption}

\begin{assumption}[Conditional excitation]\label{ass:pe}
There exist $\gamma>0$ and $h\in\mathbb{N}$ with $h\ge1$ such that, for all $j\ge0$,
\begin{equation}\label{eq:pe}
\inf_{\vartheta\in D}
\lambda_{\min}\Bigl(
\mathbb{E}\Bigl[\sum_{i=jh+1}^{(j+1)h} g_i(\vartheta)g_i(\vartheta)^{\mathrm T}\Big|\mathcal{F}_{jh}\Bigr]
\Bigr)\ge \gamma, ~~\text{a.s.}
\end{equation}
\end{assumption}

\begin{remark}\label{remark of grad}
The conditions in Assumption~\ref{ass:grad-noise} are mild and accommodate a broad class of models with either linear or nonlinear parametrizations. 
They are compatible with both convex and nonconvex loss functions and can be verified for a variety of commonly used models. 
Representative examples include linear regression with absolute-value loss~\cite{bloomfield1983lad}, logistic regression with cross-entropy loss~\cite{goodfellow2016dl}, and saturated observation models with mean-square loss~\cite{zhang2023adaptive}.

The three conditions in Assumption \ref{ass:grad-noise} play distinct technical roles.
Condition~(i) imposes uniform boundedness of the Clarke selections and second-moment bounds on the gradient noise.
Condition~(ii) is a sector-type inequality ensuring a sufficient identifiability property along the direction induced by $g_k(\vartheta)$, it is more general than the strong convexity condition.
Condition~(iii) imposes Lipschitz-type regularity. 
The bound $|\psi_k(\vartheta)|\le L_\psi|\theta_k-\vartheta|$ indicates that $\psi_k(\vartheta)=\mathbb{E}[\ell_{k+1}(\vartheta)\mid\mathcal{F}_k]$ exhibits Lipschitz-type growth with respect to the parameter mismatch (e.g., for linear regression with squared loss and $\mathbb E[\varepsilon_{k+1}\mid\mathcal F_k]=0$, $\psi_k(\vartheta)=-\phi_k^{\mathrm{T}}(\theta_k-\vartheta)$).
The Lipschitz continuity of $g_k(\cdot)$ further ensures regularity of the generalized gradient mapping. These conditions are essential for establishing parameter tracking error bounds.
\end{remark}

\begin{remark}
Assumption \ref{ass:pe} can be viewed as a nonlinear extension of the
CE condition originally introduced by~\cite{guo1990estimating}.
It provides a general stochastic data condition under which parameter-tracking
guarantees can be established.
This is substantially weaker and more general than the deterministic or
structured stochastic assumptions (e.g., i.i.d.\ or stationary and ergodic data)
commonly adopted in the Adam literature
(see, e.g.,~\cite{kingma2015adam,reddi2018convergence}). To the best of our knowledge, such a CE condition has not been employed in existing
convergence analyses of Adam-type methods.

In the linear regression case, condition~\eqref{eq:pe} coincides exactly with the classical conditional excitation condition
\begin{equation}\label{conditional pe}
\lambda_{\min}\!\Bigl(
\mathbb{E}\Bigl[\sum_{i=jh+1}^{(j+1)h}\phi_i\phi_i^{\mathrm T}\,\big|\,\mathcal{F}_{jh}\Bigr]
\Bigr)\ge \gamma>0,\quad \text{a.s.},
\end{equation}
which is known to be necessary for the stability of tracking algorithms under mild conditions~\cite{guo1994stability}. Furthermore, when the predictor in~\eqref{eq:oracle_pred} admits the form $f_k(\phi_k^{\mathrm T}\vartheta)$ and the derivative $f_k'(\cdot)$ is bounded away from zero on the compact domain induced by Assumption~\ref{ass:boundedness}, it is not difficult to show that condition~\eqref{eq:pe} becomes equivalent to~\eqref{conditional pe} and therefore depends only on the regressor sequence.
\end{remark}

We establish theoretical guarantees for parameter tracking.
\begin{theorem}\label{Thm1}
Suppose that Assumptions~\ref{ass:boundedness}--\ref{ass:pe} hold.
With appropriately chosen small hyper-parameters $\alpha$, $\beta_1$ and $1-\beta_2$,
the parameter tracking error satisfies
\begin{equation}\label{error bound}
\begin{aligned}
\mathbb{E}\|\tilde{\theta}_{k}\|^2
=&
O\left(
\frac{[\alpha^2+\beta_1^2\delta_1^2+\delta_2]^{\frac{3}{2}}}
{\alpha\delta_1}+\frac{(\alpha^2+\beta_1^2)\delta_1C_w^2}{\alpha}+\frac{C_\Delta}{\alpha\delta_1}
\right)\\
&
+O\left(\lambda^{\eta k}\mathbb E\|\tilde\theta_0\|^2\right),
\end{aligned}
\end{equation}
where $\tilde{\theta}_k = \theta_k-\hat{\theta}_k$,
$\delta_1 \triangleq 1-\beta_1$, $\delta_2 \triangleq 1-\beta_2$, $\lambda \in (0, 1)$, $\eta \triangleq \frac{1}{2h}$, and $h$ is the excitation length defined in \eqref{eq:pe}.
Moreover, $C_\Delta$ and $C_w$ are the second-moment bounds in~\eqref{delta_bound} and~\eqref{w_bound}, respectively.
\end{theorem}

The proof of Theorem \ref{Thm1} is given in Section \ref{proofsmain}.

\begin{remark}\label{remark of Thm1}
The bound~\eqref{error bound} provides an explicit and interpretable characterization of Adam’s \emph{parameter tracking} performance in time-varying and nonstationary settings. It further yields meaningful prediction guarantees, which are closely related to parameter identifiability, as identifiable parameters facilitate reliable prediction and generalization to unseen data (see, e.g.,~\cite{SIAMParameter}). Such tracking-type guarantees for Adam  have remained largely absent from the existing theory, even for the special case of constant model parameters.

We now provide some discussions of the bounds on the RHS of~\eqref{error bound}. One may try to optimize the upper-bound by selecting the triple hyperparameter ($\alpha$, $\beta_1$, $\beta_2$) over $(0,1) \times (0,1)\times (0,1)$.
Intuitively, for given stepsize $\alpha$, if both $\beta_1$ and  $1-\beta_2$ close to 0, then the upper bound  will be reduced to near $O(\alpha^2+\alpha C_w^2 + \frac{C_{\Delta}}{\alpha})$.  It is interesting to note that when we take $\beta_1=0$ and $\beta_2=1$, the Adam will become the standard LMS algorithm, in this case, the steady state term on the RHS of \eqref{error bound} will be simplified to  $O(\alpha^2+\alpha C_w^2 + \frac{C_{\Delta}}{\alpha})$, the last two main terms of this bound will coincide with that of the standard LMS (see Theorem 4 in \cite{guo1997lms}), demonstrating a clear trade off between noise sensitivity and parameter variation.  

\end{remark}

\begin{remark}\label{remark:param-choice}
An explicit admissible choice of the hyper-parameters in Theorem~\ref{Thm1}
can be constructed as follows: 
\begin{subequations}\label{eq:param_bounds}
\begin{align}
&0<\alpha \le \min\Bigl\{\bar\alpha,\alpha_h,\tfrac{c_0}{4h\mu M_g^2},1\Bigr\}, \\
& 0\leq\beta_1 \le \min\Bigl\{\bar{\beta}_1,\sqrt{\tfrac{(1-\bar{\beta}_1)\alpha\gamma_0}{3(4h+\gamma_0)\mathcal{U}}}\Bigr\}, \\
&0<1 -\beta_2 \leq\min\Bigl\{1-\underline{\beta_2},\beta_{2,h},
\Bigl[\tfrac{\alpha(1-\beta_1)c_0\gamma_0}{3(4h+\gamma_0)M_{\ell}M_g}\Bigr]^2\Bigr\} ,
\end{align}
\end{subequations}
where $0<\bar{\beta}_1, \underline{\beta_2}<1$ are arbitrary constants, and
\begin{align}
\gamma_0 & \triangleq \frac{\mu\gamma}{4(M_{\ell}M_g+c_0)},  \label{gammma0}\\
\mathcal{U}&\triangleq
\frac{2L_{\psi}^2M_g^2}{c_0^2}
+\frac{L_{\psi} M_g \sqrt{M_{\ell} M_g + c_0}}{c_0^{3/2}}
+\frac{2}{\sqrt{\underline{\beta_2}}}, \label{mathcalU} \\
\bar\alpha &\triangleq
\frac{(1-\bar{\beta}_1)\gamma_0}{6\gamma_0^2+24h^2+3(4h+\gamma_0)\mathcal{U}}, \\
\alpha_h &\triangleq
\sqrt{\dfrac{3c_0^3\gamma_0}{2\mu h(h+1)(2h+1)L_g^2M_{\ell}^2M_g^2}},\\
\beta_{2,h} &\triangleq
\dfrac{2c_0^3\gamma_0}{\mu\,h(h+1)M_{\ell}^2M_g^4}.
\end{align}
\end{remark}

\subsection{Prediction performance}

So far, we have established parameter-tracking guarantees for Adam, which may secure the performance of prediction for general data sequence. These guarantees, however, rely on somewhat stringent data excitation conditions. Fortunately, such conditions turn out to be not necessary for guaranteed adaptive prediction, as will be shown in Theorem \ref{Thm2} below.

We need the following prediction-oriented sector assumption which is weaker than the usual convexity condition.

\begin{assumption}\label{ass:regret-sector}
Let $g_k(\vartheta)$ be the Clarke selection in~\eqref{eq:subgradient_def} and
$\psi_k(\vartheta)$ and $w_{k+1}(\vartheta)$ be defined in~\eqref{decomposition noise}.
Under Assumption~\ref{ass:boundedness}, there exist constants $M_{\psi}, M_g,C_w,\kappa>0$, such that for all $k\ge0$ and all $\vartheta\in D$, a.s.,
\begin{equation} \label{CW2}
|\psi_k(\vartheta)|\le M_\psi, ~~\|g_k(\vartheta)\|\le M_g,
~~
\mathbb{E}|w_{k+1}(\vartheta)|^2\le C_w^2,
\end{equation}
and, for some $p\in[1,2]$,
\begin{equation}\label{eq:sector_regret}
\psi_k(\vartheta)\, g_k(\vartheta)^{\mathrm T}(\vartheta-\theta_k)
\ge
\kappa \Bigl(\mathcal{L}\bigl(f_k(\phi_k,\theta_k),\,f_k(\phi_k,\vartheta)\bigr)\Bigr)^{\frac{2}{p}},
\end{equation}
\end{assumption}

\begin{remark}
Assumption~\ref{ass:regret-sector} will be used for prediction-error analysis and is strictly weaker than Assumption~\ref{ass:grad-noise} in two respects.
First, it allows the Clarke selection $\ell_{k+1}(\vartheta)=\psi_k(\vartheta)+w_{k+1}(\vartheta)$ to be almost surely unbounded, since the noise term $w_{k+1}(\vartheta)$ need not be bounded almost surely and is only required to have a uniformly bounded second moment.
Second, it imposes no Lipschitz-type regularity on $\psi_k(\cdot)$ or $g_k(\cdot)$, as such regularity is typically needed for parameter tracking but not for  output prediction guarantees.

Condition~\eqref{eq:sector_regret} is inspired by~\cite{liu2026gradient} and is in the spirit of the gradient--loss coercivity
conditions widely used in online optimization and learning (see, e.g.,~\cite{ShalevShwartz}). It provides a prediction-error descent-type control, enabling prediction-error guarantees without requiring loss convexity.
The parameter $p\in[1,2]$ is introduced to accommodate a broader class of prediction losses by allowing different
growth geometries, with the exponent $2/p$ mapping $p$-power losses (e.g., $|y_{k+1}-\hat y_{k+1}|^p$) to a common
squared-error--type scale, thereby covering the absolute-loss regime ($p=1$) and the quadratic-loss regime ($p=2$).
\end{remark}

\begin{theorem}\label{Thm2}
Suppose that Assumptions~\ref{ass:boundedness} and~\ref{ass:regret-sector} hold. Then,
\begin{equation}\label{regret}
\begin{aligned}
&\frac{1}{n}\sum_{k=1}^{n}\mathbb{E}\mathcal{L}\!\left(f_k(\phi_k,\theta_k),\,f_k(\phi_k,\hat{\theta}_k)\right)\\
=&
O\Biggl(
\Biggl[
\frac{\alpha^2+\beta_1^2\delta_1^2+\sqrt{\delta_2}}{\alpha\delta_1}
+\frac{(\alpha^2+\sqrt{\delta_2})C_w}{\alpha\delta_1}\\
&\qquad+\frac{(\alpha^2+\beta_1^2)\delta_1C_w^2}{\alpha}
+\frac{(1+C_w)C_\Delta}{\alpha\delta_1}
\Biggr]^{\frac{p}{2}}
\Biggr)\\
&\quad+O\!\left((\alpha\delta_1 n)^{-\frac{p}{2}}\right).
\end{aligned}
\end{equation}
provided that 
$
    \alpha^2+\beta_1^2<\sqrt{\beta_2},
$
where $\delta_1 = 1-\beta_1$, $\delta_2 = 1-\beta_2$, $C_\Delta$ and $C_{w}$ are the constants defined in \eqref{delta_bound} and \eqref{CW2} respectively, $p\in[1,2]$ is defined in Assumption~\ref{ass:regret-sector}.
\end{theorem}

The proof of Theorem \ref{Thm2} is given in Section \ref{proofsmain}.

\begin{remark}\label{regret_intuitive}
The bound~\eqref{regret} provides an output prediction guarantee without requiring any excitation condition on the data. It exhibits a qualitative dependence on the hyperparameters, gradient noise, and parameter variation,
that is different from \eqref{error bound}. The optimization of the RHS of \eqref{regret} over the triple $(\alpha$, $\beta_1$, $\beta_2)$ is a complicated issue, and we only give some heuristic discussions here. It is not difficult to see intuitively that if $\delta_2=1-\beta_2$ and  $\beta_1^2\alpha^{-1}$ decrease, then the upper bound will also decrease. Hence we may take $\beta_2$ to be close to 1 and $\beta_1^2\alpha^{-1}$ to be close to 0. Consequently, the upper bound will be simplified to $O(\alpha (1+C_w+C_w^2)+ \frac{(1+C_w)C_{\Delta}}{\alpha})$, which can be minimized by taking $\alpha = \sqrt{(1+C_w)(1+C_w+C_w^2)C_{\Delta}}$. 

\end{remark}

\begin{remark}
We now compare the output prediction bound~\eqref{regret} with representative analyses of Adam-type methods developed for static learning regimes with time-invariant system parameters under i.i.d.\ or stationary ergodic data. 

First, the dependence on the second momentum parameter $\beta_2$ is at most of order $\sqrt{\delta_2}=\sqrt{1-\beta_2}$ in \eqref{regret}, which vanishes as $\beta_2\to1$ and is therefore much better than the common $(1-\beta_2)^{-1}$-type factors in existing prediction bounds
(e.g., $(1-\beta_2)^{-1/2}$; see~\cite{reddi2018convergence, alacaoglu2020regret}). Consequently, \eqref{regret} remains
well-conditioned for  $\beta_2$ close to 1, consistent with both the observations in ~\cite{Zhang2022AdamConverge} and our simulations in Section \ref{simulation}.

Second, whereas most existing analyses are confined to static regimes with time-invariant parameters and gradient noise only, \eqref{regret} remains valid in time-varying and nonstationary settings where gradient noise and parameter drift coexist.
The bound separates and quantifies these two sources of error within a single guarantee, thereby extending output prediction theory well beyond time-invariant parameter models with i.i.d. datasets.

Overall, our results lead to sharper hyperparameter scalings and a clearer theoretical characterization of Adam-type methods in the general time-varying and nonstationary scenarios.
\end{remark}

\section{Proofs of the main theorems} \label{proofsmain}

In this section, we present the proofs of the main theorems of this paper.
For convenience of analysis, we first introduce some notation.
Let
\begin{equation}\label{def Ij}
\mathcal{I}_j \triangleq \{j\bar h,\dots,(j+1)\bar h-1\},~ j\ge 0,
~
\bar h \triangleq 2h,~ a \triangleq \alpha\delta_1,
\end{equation}
where $h$ is defined in Assumption~\ref{ass:pe}, $\delta_1 =1-\beta_1$.
We further introduce the following quantities, for ~$k\ge 0$,
\begin{equation}\label{notations 1}
\begin{aligned}
\psi_k &\triangleq \psi_k(\hat{\theta}_k),
&
w_{k+1} &\triangleq w_{k+1}(\hat{\theta}_k),\\
\Phi_{k+1} &\triangleq \ell_{k+1}^2 \operatorname{diag}\{g_k\odot g_k\},
&
S_k &\triangleq \mu\,\bar V_{k}^{-\frac{1}{4}} g_k g_k^{\mathrm T}
\bar V_{k}^{-\frac{1}{4}}.
\end{aligned}
\end{equation}

The purpose of this section is to present the main proof structure. 
Several technical estimates are stated as lemmas. Their proofs rely on 
lengthy but standard perturbation bounds for the adaptive scaling matrix, 
the non-expansiveness of the weighted projection, and the conditional 
excitation condition. To keep this short version concise, we omit these 
constant-level details and focus on the Lyapunov construction and the 
resulting block-recursive estimates.

\begin{lemma}\label{project}
(\cite{cheney2001})
Let $\Pi_Q(\cdot)$ be the projection mapping defined in
Definition~\ref{prodef}. Then, for any $x_1,x_2\in\mathbb{R}^{d_2}$,
\begin{equation}
\|\Pi_Q(x_1)-\Pi_Q(x_2)\|_Q
\;\le\;
\|x_1-x_2\|_Q.
\end{equation}
\end{lemma}

\begin{lemma}\label{lemma1}
Under the conditions of Theorem \ref{Thm1}, one can deduce that 
\begin{equation}\label{PE2}
\mathbb{E}\Bigl[\sum_{i\in\mathcal{I}_j}S_i\Big\rvert \mathcal{F}_{j \bar{h}}\Bigr]\succeq \gamma_0 I\succeq \mathbf{0}, ~j\geq 0, ~ \text { a.s.},
\end{equation}
where $\gamma_0 = \frac{\mu\gamma}{4(M_{\ell}M_g+c_0)}$, $\bar{h} =2h$.
\end{lemma}

\begin{lemma}\label{lemma2}
Under Assumptions~\ref{ass:boundedness}--\ref{ass:grad-noise}, for any $j>i\ge 0$,
\begin{equation}\label{difference theta}
\begin{aligned}
\mathbb{E}\|\tilde{\Theta}_j-\tilde{\Theta}_i\|^2
\le\;&
3(M_{\ell}M_g + c_0) (j-i)^2
\Bigl(C_{\Delta}^2 + \tfrac{\alpha^2}{c_0^2}M_{\ell}^2M_g^2\Bigr) \\
&+(1-\beta_2^{j-i})M_{\ell}^2M_g^2
\Bigl(\tfrac{3L^2}{c_0} + \beta_1^2\tfrac{M_{\ell}^2M_g^2}{2c_0^3}\Bigr)\\
&+\beta_1^2(1-\beta_1^{j-i})^2\,
\tfrac{8M_{\ell}^2M_g^2}{c_0},
\end{aligned}
\end{equation}
where $\tilde{\Theta}_j$ is defined in \eqref{Theta definition}.
\end{lemma}

\begin{lemma}\label{lemma5}
Under the conditions of Theorem~\ref{Thm1}, there exist constants
$C_b>0$ and $\gamma_0>0$ such that, for all $j\ge0$,
\begin{equation}\label{block_contraction_short}
\begin{aligned}
\mathbb E\|\tilde\Theta_{(j+1)\bar h}\|^2
\leq&
(1-a\gamma_0)\mathbb E\|\tilde\Theta_{j\bar h}\|^2\\
&+C_b\Bigl[
(\alpha+\beta_1^2+\sqrt{\delta_2})
(C_\Delta^2+\alpha^2+\beta_1^2\delta_1^2+\delta_2)\\
&\qquad
+C_\Delta+C_\Delta^2
+(\alpha^2+\beta_1^2)\delta_1^2C_w^2
\Bigr],
\end{aligned}
\end{equation}
where $\bar h=2h$, $a=\alpha\delta_1$, $\delta_1=1-\beta_1$,
and $\delta_2=1-\beta_2$.
\end{lemma}

\noindent\hspace{1em}{\textbf{\itshape Proof of Theorem \ref{Thm1}:}}
Using \eqref{bar_Vk}, \eqref{update rule}, together with
Algorithm~\ref{Algorithm}, it follows under Assumption~\ref{ass:grad-noise} that,
for all $k\ge 0$,
\begin{equation}\label{critical 1}
    \begin{aligned}
        &\begin{bmatrix}
            \bar{V}_{k+1}^{\frac{1}{4}}(\theta_{k+1} - \hat{\theta}_k
            + \alpha \bar{V}_{k+1}^{-\frac{1}{2}}m_{k+1})\\
            \beta_1\bar{V}_{k+1}^{-\frac{1}{4}}m_{k+1}
        \end{bmatrix}\\
        =& \underbrace{\begin{bmatrix}
             \bar{V}_{k+1}^{\frac{1}{4}}\tilde{\theta}_{k}
             +a \bar{V}_{k+1}^{-\frac{1}{4}}g_k\psi_k
             + \alpha\beta_1\bar{V}_{k+1}^{-\frac{1}{4}}m_{k}\\
             \beta_1^2\bar{V}_{k+1}^{-\frac{1}{4}}m_{k}
             +\beta_1\delta_1\bar{V}_{k+1}^{-\frac{1}{4}}g_k\psi_k
        \end{bmatrix}}_{\mathcal T_{k+1}}\\
        &+
        \underbrace{\begin{bmatrix}
            a\bar{V}_{k+1}^{-\frac{1}{4}}g_k\\
            \beta_1\delta_1\bar{V}_{k+1}^{-\frac{1}{4}}g_k
        \end{bmatrix}}_{\mathcal{N}_{k+1}}w_{k+1}
        + \underbrace{\begin{bmatrix}
             \bar{V}_{k+1}^{\frac{1}{4}}\\
             \mathbf{0}
        \end{bmatrix}}_{\mathcal{P}_{k+1}}\Delta_{k+1},
    \end{aligned}
\end{equation}
where $a=\alpha\delta_1$, $\delta_1=1-\beta_1$, and
$\delta_2=1-\beta_2$.

We choose the stochastic Lyapunov function
\[
V_k^\Theta \triangleq \|\tilde{\Theta}_k\|^2,\qquad k\ge0.
\]
Then, by \eqref{bar_Vk}, \eqref{update rule}, \eqref{Theta definition},
\eqref{critical 1}, and Lemma~\ref{project}, we obtain
\begin{equation}\label{Lyapunov equation}
\begin{aligned}
    \|\tilde{\Theta}_{k+1}\|^2
    \leq&\;
    \mathcal{T}_{k+1}^{\mathrm{T}}\mathcal{T}_{k+1}
    +2\mathcal{T}_{k+1}^{\mathrm{T}}\mathcal{N}_{k+1}w_{k+1}
    +2\mathcal{T}_{k+1}^{\mathrm{T}}\mathcal{P}_{k+1}\Delta_{k+1} \\
    &+2\mathcal{N}_{k+1}^{\mathrm{T}}\mathcal{N}_{k+1}w_{k+1}^2
    +2\Delta_{k+1}^{\mathrm{T}}\mathcal{P}_{k+1}^{\mathrm{T}}
      \mathcal{P}_{k+1}\Delta_{k+1}.
\end{aligned}
\end{equation}
Here the last inequality follows by expanding the square and applying
$2ab\le a^2+b^2$ to the cross term between
$\mathcal{N}_{k+1}w_{k+1}$ and $\mathcal{P}_{k+1}\Delta_{k+1}$.

By applying the block contraction estimate in Lemma~\ref{lemma5} to
the one-step Lyapunov inequality \eqref{Lyapunov equation}, we obtain,
for all $j\ge0$,
\begin{equation}\label{very important compression}
\begin{aligned}
\mathbb E\|\tilde\Theta_{(j+1)\bar h}\|^2
\leq&
(1-a\gamma_0)\mathbb E\|\tilde\Theta_{j\bar h}\|^2\\
&+C_b\Bigl[
(\alpha+\beta_1^2+\sqrt{\delta_2})
(C_\Delta^2+\alpha^2+\beta_1^2\delta_1^2+\delta_2)\\
&\qquad
+C_\Delta+C_\Delta^2
+(\alpha^2+\beta_1^2)\delta_1^2C_w^2
\Bigr].
\end{aligned}
\end{equation}
For brevity, the constant-level perturbation estimates leading to
\eqref{very important compression} are omitted in this short version.

Denote the second term on the right-hand side of
\eqref{very important compression} by $B_{\alpha,\beta,\Delta,w}$, i.e.,
\begin{equation}\label{B_def}
\begin{aligned}
B_{\alpha,\beta,\Delta,w}
\triangleq C_b\Bigl[
&(\alpha+\beta_1^2+\sqrt{\delta_2})
(C_\Delta^2+\alpha^2+\beta_1^2\delta_1^2+\delta_2)\\
&+C_\Delta+C_\Delta^2
+(\alpha^2+\beta_1^2)\delta_1^2C_w^2
\Bigr].
\end{aligned}
\end{equation}
Since $1-a\gamma_0\in(0,1)$, iterating
\eqref{very important compression} gives, for any $j\ge1$,
\begin{equation}\label{eq:bh_iter}
\begin{aligned}
\mathbb{E}\|\tilde{\Theta}_{j\bar h}\|^2
\leq&
(1-a\gamma_0)^j\mathbb{E}\|\tilde{\Theta}_{0}\|^2
+B_{\alpha,\beta,\Delta,w}
\sum_{t=0}^{j-1}(1-a\gamma_0)^t\\
\leq&
(1-a\gamma_0)^j\mathbb{E}\|\tilde{\Theta}_{0}\|^2
+\frac{B_{\alpha,\beta,\Delta,w}}{a\gamma_0}.
\end{aligned}
\end{equation}

For a general time $k\ge0$, let
\[
j_1\triangleq \max\{j\ge0: j\bar h\le k\}.
\]
By Young's inequality, for any $\nu>0$,
\begin{equation}\label{eq:Theta_split_young}
\begin{aligned}
\mathbb E\|\tilde\Theta_k\|^2
&=\mathbb E\bigl\|
\tilde\Theta_{j_1\bar h}
+(\tilde\Theta_k-\tilde\Theta_{j_1\bar h})
\bigr\|^2\\
&\le
(1+\nu)\mathbb E\|\tilde\Theta_{j_1\bar h}\|^2
+\Bigl(1+\frac{1}{\nu}\Bigr)
\mathbb E\|\tilde\Theta_k-\tilde\Theta_{j_1\bar h}\|^2 .
\end{aligned}
\end{equation}
Since $0\le k-j_1\bar h\le \bar h-1$, Lemma~\ref{lemma2} implies that
there exists a constant $C_\Theta>0$ such that
\begin{equation}\label{within_block_short}
\mathbb E\|\tilde\Theta_k-\tilde\Theta_{j_1\bar h}\|^2
\le
C_\Theta
\bigl(C_\Delta^2+\alpha^2+\beta_1^2\delta_1^2+\delta_2\bigr).
\end{equation}
Combining \eqref{eq:bh_iter}--\eqref{within_block_short}, and taking
$\nu=1$ for simplicity, yields
\begin{equation}\label{differnence main}
\begin{aligned}
\mathbb E\|\tilde\Theta_k\|^2
=&
O\!\left((1-a\gamma_0)^{j_1}
\mathbb E\|\tilde\Theta_0\|^2\right)
+
O\!\left(\frac{B_{\alpha,\beta,\Delta,w}}{a}\right)\\
&+
O\!\left(
C_\Delta^2+\alpha^2+\beta_1^2\delta_1^2+\delta_2
\right).
\end{aligned}
\end{equation}

By \eqref{bar_Vk} and \eqref{Theta definition}, and since $m_0=0$ and
$V_0=\mathbf 0$, we have
\begin{equation}\label{theta_0}
\|\tilde{\Theta}_{0}\|^2
=
c_0\|\tilde{\theta}_0\|^2,
\qquad
\|\tilde{\Theta}_{k}\|^2
\ge
c_0\|\tilde{\theta}_{k}\|^2 .
\end{equation}
Moreover, for the admissible hyperparameter range considered in the theorem,
$\beta_1$ is bounded away from $1$, and hence
\[
\alpha+\beta_1^2+\sqrt{\delta_2}
=
O\!\left(
\sqrt{\alpha^2+\beta_1^2\delta_1^2+\delta_2}
\right).
\]
Since $a=\alpha\delta_1\le1$, the last term in
\eqref{differnence main} is absorbed into the corresponding
$O(\cdot/a)$ terms. In addition, because $\theta_k\in D$ and $D$ is compact,
the drift level is bounded, so the term $C_\Delta^2$ can be absorbed into
$C_\Delta$ up to a change of constants. Therefore, using
$a=\alpha\delta_1$ and setting
\[
\lambda\triangleq 1-a\gamma_0\in(0,1),
\]
we obtain
\begin{equation}\label{eq:theta_bh_final}
\begin{aligned}
\mathbb E\|\tilde\theta_{k}\|^2
\le&
O\left(
\frac{[\alpha^2+\beta_1^2\delta_1^2+\delta_2]^{\frac{3}{2}}}
{\alpha\delta_1}
\right)\\
&+
O\left(
\frac{(\alpha^2+\beta_1^2)\delta_1}{\alpha}C_w^2
\right)
+
O\left(
\frac{C_\Delta}{\alpha\delta_1}
\right)\\
&+
O\left(
\lambda^{j_1}\mathbb E\|\tilde\theta_0\|^2
\right).
\end{aligned}
\end{equation}
 This completes the proof.
$\hfill\blacksquare$

\noindent\hspace{1em}{\textbf{\itshape Proof of Theorem \ref{Thm2}:}} Under the same Lyapunov function $\|\tilde{\Theta}_k\|^2$, $k\ge0$, and following the same line of analysis as in \eqref{Lyapunov equation}, we have
\begin{equation}\label{important lyapunov}
    \begin{aligned}
&\|\tilde{\Theta}_{k+1}\|^2\\
\leq & \tilde{\theta}_k^{\mathrm T}\bar V_{k+1}^{\frac12}\tilde{\theta}_k
 +  2a\psi_kg_k^{\mathrm T}\tilde{\theta}_k
 + 2\alpha\beta_1\tilde{\theta}_k^{\mathrm T}m_{k}\\
 &+\bigl(\alpha^2+\beta_1^2\bigr)\delta_1^2\psi_k^2
   g_k^{\mathrm T}\bar V_{k+1}^{-\frac12}g_k
 + \bigl(\alpha^2+\beta_1^2\bigr)\beta_1^2
   m_{k}^{\mathrm T}\bar V_{k+1}^{-\frac12}m_{k} \\
&+ 2\bigl(\alpha^2 + \beta_1^2\bigr) \beta_1\delta_1\psi_k
   g_k^{\mathrm T}\bar V_{k+1}^{-\frac12}m_{k}
    + 2a g_k^{\mathrm T}\tilde\theta_k  w_{k+1}\\
&+2(\alpha^2+\beta_1^2)\delta_1^2 \psi_k
g_k^{\mathrm T}\bar V_{k+1}^{-\frac12}g_kw_{k+1} \\
&+2(\alpha^2+\beta_1^2)\beta_1\delta_1 
m_{k}^{\mathrm T}\bar V_{k+1}^{-\frac12}g_k w_{k+1}\\
& + 2\tilde\theta_k^{\mathrm T}\bar V_{k+1}^{\frac12}\Delta_{k+1}
+2a\psi_k g_k^{\mathrm T}\Delta_{k+1}
+2\alpha\beta_1m_{k}^{\mathrm T}\Delta_{k+1}\\
& + 2(\alpha^2+\beta_1^2)\delta_1^2g_k^{\mathrm{T}}\bar{V}_{k+1}^{-\frac{1}{2}}g_kw_{k+1}^2 +2 \Delta_{k+1}^{\mathrm{T}}\bar{V}_{k+1}^{\frac{1}{2}}\Delta_{k+1}.
    \end{aligned}
\end{equation}

We analyze each term on the RHS of \eqref{important lyapunov} separately, so as to establish a recursive Lyapunov inequality without requiring any data excitation condition. The overall argument parallels that of the proof of Theorem~\ref{Thm1}.

Since $\alpha^2+\beta_1^2<\sqrt{\beta_2}$, we can choose a sufficiently large constant
$r>0$ such that
\[
\frac{1}{r}+\frac{(r+2)\bigl(\alpha^2+\beta_1^2\bigr)}{r\sqrt{\beta_2}}\le 1.
\]
 Then we have
\begin{equation}\label{lyapunov2}
    \begin{aligned}      
    &\|\tilde{\Theta}_{k+1}\|^2 \\
    \leq& \|\tilde{\Theta}_{k}\|^2 -2a\kappa \mathcal{L}_k^{\frac{2}{p}}+ \sqrt{\delta_2}\tilde{\theta}_k^{\mathrm{T}}(\Phi_{k+1}^{\frac{1}{2}}+c_0I) \tilde{\theta}_k\\ &
    +  r \alpha^2 \tilde{\theta}_k^{\mathrm{T}}\bar{V}_{k}^{\frac{1}{2}}\tilde{\theta}_k 
       + (r+2)\bigl(\alpha^2+\beta_1^2\bigr)\delta_1^2\psi_k^2
   g_k^{\mathrm T}\bar V_{k+1}^{-\frac12}g_k \\
   & + 2a g_k^{\mathrm T}\tilde\theta_k w_{k+1} +(r+3)(\alpha^2+\beta_1^2)\delta_1^2g_k^{\mathrm{T}}\bar{V}_{k+1}^{-\frac{1}{2}}g_kw_{k+1}^2 \\
   &+ 2\tilde\theta_k^{\mathrm T}\bar V_{k+1}^{\frac12}\Delta_{k+1}+2a\psi_k g_k^{\mathrm T}\Delta_{k+1}
+2\alpha\beta_1m_{k}^{\mathrm T}\Delta_{k+1} \\
&+2 \Delta_{k+1}^{\mathrm{T}}\bar{V}_{k+1}^{\frac{1}{2}}\Delta_{k+1}, ~\text{a.s.},
    \end{aligned}
\end{equation}
where we have used the fact that $\Bigl[\frac{1}{r}+\frac{(r+2)\bigl(\alpha^2+\beta_1^2\bigr)}{r\sqrt{\beta_2}}\Bigr]\leq 1$. 

To analyze \eqref{lyapunov2}, we further introduce several inequalities in expectation.

With Assumption \ref{ass:regret-sector}, \eqref{decomposition noise} and \eqref{notations 1}, and noticing that $\{w_{k}, \mathcal{F}_k\}$ is the martingale difference sequence and $\psi_k$ is
$\mathcal{F}_k$-measurable, we have
\begin{equation}\label{89}
\begin{aligned}
    \mathbb{E}\|\Phi_{k+1}^{\frac{1}{2}}\|  \leq 
     M_g\sqrt{M_{\psi}^2+C_w^2}.
\end{aligned}
\end{equation} 
Then, combining $\|\tilde{\theta}_k\|\le 2L$ under Assumption~\ref{ass:boundedness}, we have
\begin{equation}
    \mathbb{E} \tilde{\theta}_k^{\mathrm{T}}(\Phi_{k+1}^{\frac{1}{2}}+c_0I) \tilde{\theta}_k  \leq 4L^2\Bigl(M_g\sqrt{M_\psi^2+C_w^2}+c_0\Bigr),
\end{equation}
where we used that $\Phi_k$ is diagonal.

Similarly,  we have
\begin{equation}\label{cross 0}
\mathbb{E}g_k^{\mathrm T}\tilde\theta_k\, w_{k+1}
=
\mathbb{E}\bigl[g_k^{\mathrm T}\tilde\theta_k\,\mathbb{E}\!\left[w_{k+1}\mid\mathcal{F}_k\right]\bigr]
=0.
\end{equation}

Since $V_k=\beta_2 V_{k-1}+\delta_2\Phi_k$ with $V_{0}=0$, it follows from \eqref{notations 1} that, similarly to \eqref{89}, we have
\begin{equation}\label{eq:Vk_bound}
\begin{aligned}
\mathbb E\|V_k\|
\leq  M_g^2(M_\psi^2+C_w^2).
\end{aligned}
\end{equation}
Consequently, since $\bar V_k^{\frac{1}{2}}=V_k^{\frac{1}{2}}+c_0 I$ is diagonal, we have
\begin{equation}\label{eq:Vkhalf_bound}
\mathbb E\|\bar V_k^{\frac{1}{2}}\|
\le \mathbb E\|V_k^{\frac{1}{2}}\|+c_0
\le M_g\sqrt{M_\psi^2+C_w^2}+c_0.
\end{equation}
Since $\|\bar V_k\|
\leq \bigl(\|V_k\|^{\frac{1}{2}}+c_0\bigr)^2 $,
taking expectations and using Jensen's inequality yield
\begin{equation}\label{eq:EbarV_norm}
\mathbb E\|\bar V_k\|
\le \Bigl((\mathbb E\|V_k\|)^{\frac{1}{2}}+c_0\Bigr)^2\le \bigl(M_g\sqrt{M_\psi^2+C_w^2}+c_0\bigr)^2.
\end{equation}
With \eqref{eq:EbarV_norm} and $\|\tilde{\theta}_k\|\le 2L$, we have
\begin{equation}
    \mathbb{E}\tilde{\theta}_k^{\mathrm{T}}\bar{V}_{k}^{\frac{1}{2}} \tilde{\theta}_k \leq  \mathbb{E} \|\bar{V}_{k}^{\frac{1}{2}}\|\|\tilde{\theta}_k\|^2 \le 4L^2\bigl(M_g\sqrt{M_\psi^2+C_w^2}+c_0\bigr).
\end{equation}

Moreover, under Assumption \ref{ass:regret-sector}, and with \eqref{bar_Vk}, $g_i^{\mathrm T}\bar V_{i+1}^{-\frac12}g_i \leq \frac{M_g^2}{c_0}$ a.s.,
\begin{equation}
    \mathbb{E}\psi_k^2
   g_k^{\mathrm T}\bar V_{k+1}^{-\frac12}g_k \leq \frac{M_{\psi}^2M_g^2}{c_0}.
\end{equation}
\begin{equation}
    \mathbb{E}
   g_k^{\mathrm T}\bar V_{k+1}^{-\frac12}g_k w_{k+1}^2 \leq \frac{M_g^2}{c_0} \mathbb{E} w_{k+1}^2 \leq \frac{M_g^2C_w^2}{c_0}.
\end{equation}

With \eqref{eq:EbarV_norm}, using $\|\tilde{\theta}_k\|\le 2L$, $\mathbb{E}\|\Delta_{k+1}\|^2\le C_\Delta^2$, and noting that $\bar V_k^{\frac{1}{2}}$ is diagonal and positive definite, we have
\begin{equation}
\begin{aligned}
    \mathbb{E}\tilde{\theta}_k^{\mathrm{T}}\bar{V}_{k+1}^{\frac{1}{2}}\Delta_{k+1} 
    \leq 2LC_{\Delta}\bigl(M_g\sqrt{M_\psi^2+C_w^2}+c_0\bigr).
\end{aligned}
\end{equation}

Similarly, we have
\begin{equation}
    \mathbb{E} \psi_k g_k^{\mathrm T}\Delta_{k+1} \leq M_{\psi}M_gC_{\Delta}.
\end{equation}

Similar to the analysis of \eqref{eq:Vk_bound}, we have
\begin{equation}
\begin{aligned}
\mathbb E\|m_k\|^2
\le M_g^2(M_{\psi}^2 + C_w^2),
\end{aligned}
\end{equation}
where we use the weighted Cauchy--Schwarz inequality.

Then we have
\begin{equation}
\begin{aligned}
    \mathbb{E}m_{k}^{\mathrm{T}}\Delta_{k+1}  \leq& \sqrt{\mathbb{E}\|m_{k}\|^2\mathbb{E}\|\Delta_{k+1}\|^2}\\
    \leq& M_gC_{\Delta}\sqrt{M_{\psi}^2 + C_w^2}.
\end{aligned}
\end{equation}

Moreover, 
\begin{equation}\label{last Theta}
\begin{aligned}
    \mathbb{E}\Delta_{k+1}^{\mathrm{T}}\bar{V}_{k+1}^{\frac{1}{2}}\Delta_{k+1}   \leq &(E\|\bar{V}_{k+1}\|)^{\frac{1}{2}}(E\|\Delta_{k+1}\|^4)^{\frac{1}{2}}\\
    \leq & 2LC_{\Delta}\bigl(M_g\sqrt{M_\psi^2+C_w^2}+c_0\bigr),
\end{aligned}
\end{equation}
where we used the fact that $\|\Delta_{k+1}\|^4 \leq 4L^2\|\Delta_{k+1}\|^2$.

Combining \eqref{lyapunov2}--\eqref{last Theta}, taking expectations on both sides of \eqref{lyapunov2} and summing the resulting inequality from $k=1$ to $n$, we have
\begin{equation}\label{leijia}
    \begin{aligned}
        \mathbb{E}\|\tilde{\Theta}_{n+1}\|^2 \leq &\mathbb{E}\|\tilde{\Theta}_1\|^2 -\sum_{i=1}^{n}2a\kappa\mathbb{E}\mathcal{L}_i^{\frac{2}{p}}\\
        &+ 4nL^2\mathcal{M}_0\Bigl(\sqrt{\delta_2}+ r\alpha^2\Bigr) \\
&+ n(\alpha^2+\beta_1^2)\delta_1^2\,
\frac{M_g^2}{c_0}\Bigl[(r+2)M_\psi^2+(r+3)C_w^2\Bigr] \\
&+ 2n(\alpha\beta_1+4L)\mathcal{M}_0C_\Delta
+ 2na M_\psi M_g C_\Delta,
    \end{aligned}
\end{equation}
where $ \mathcal{M}_0 =M_g\sqrt{M_\psi^2+C_w^2}+c_0 .$

Moreover, one can deduce that
\begin{equation}\label{inequality}
    \sum_{i=1}^{n}\mathbb{E}\mathcal{L}_i \leq n^{1-\frac{p}{2}} \mathbb{E}\Bigl(\sum_{i=1}^{n}\mathcal{L}_i^{\frac{2}{p}}\Bigr)^{\frac{p}{2}} \leq n^{1-\frac{p}{2}}\Bigl(\mathbb{E}\sum_{i=1}^{n}\mathcal{L}_i^{\frac{2}{p}}\Bigr)^{\frac{p}{2}},
\end{equation}
where $p \in [1,2]$ is defined in Assumption \ref{ass:regret-sector}.

Combining \eqref{leijia} and \eqref{inequality}, using $(x_1+x_2)^{\frac p2}\le x_1^{\frac p2}+x_2^{\frac p2}$ for $x_1,x_2\ge0$ and $p\in[1,2]$, and noting that $\mathbb{E}\|\tilde{\Theta}_1\|^2$ is uniformly bounded, we obtain \eqref{regret}.
$\hfill\blacksquare$

\section{Experiments}\label{simulation}
We illustrate the theoretical findings through experiments on both synthetic and real-world datasets.

\noindent
\textbf{Numerical Experiments.}  

We consider an online linear regression problem with drifting parameters and nonstationary regressors. The data-generating process is given by
\begin{equation}\label{eq:jump-system}
\left\{
\begin{aligned}
&\phi_{k+1} = A_k \phi_k + v_k,\\[1mm]
&\theta_k = \Pi_{[-5,5]^5}(\theta_{k-1} + \eta_k),\\[1mm]
&y_k = \phi_k^{\mathrm{T}} \theta_k + \varepsilon_k ,
\end{aligned}
\right.
\end{equation}
where $\Pi_{[-5,5]^5}(\cdot)$ denotes the projection onto the hypercube $[-5,5]^5$, and
$\theta_0 \sim \mathrm{N}(0, I_5), 
~
\eta_k \overset{\mathrm{i.i.d.}}{\sim} \mathrm{N}(0, 0.01^2 I_5),
$
where $\mathrm{N}(\cdot,\cdot)$ denotes the Gaussian distribution. This model introduces gradual temporal drift in the parameters.
The regressor $\{\phi_k\}$ is generated by a time-varying  autoregressive process with diagonal matrix $A_k = \mathrm{diag} (0.68 + 0.3 \sin\left(\frac{2\pi k}{50000}\right),0.14,0.11,0.85,0.62)$,  $v_k \overset{\mathrm{i.i.d.}}{\sim} \mathcal{N}\left(0,0.25 I_5\right)$.  The observation noise 
$
\varepsilon_k \sim \mathcal{N}(0, 0.36).
$
It can be verified that the regressor sequence $\{\phi_k\}$ is nonstationary but satisfies the proposed data excitation condition~\eqref{eq:pe} with $h=1$ and $\gamma=0.25$.
The presence of drifting parameters together with correlated nonstationary regressors results in a challenging online learning setting.

\begin{figure}[htbp]
\begin{center}
\includegraphics[width=1\linewidth]{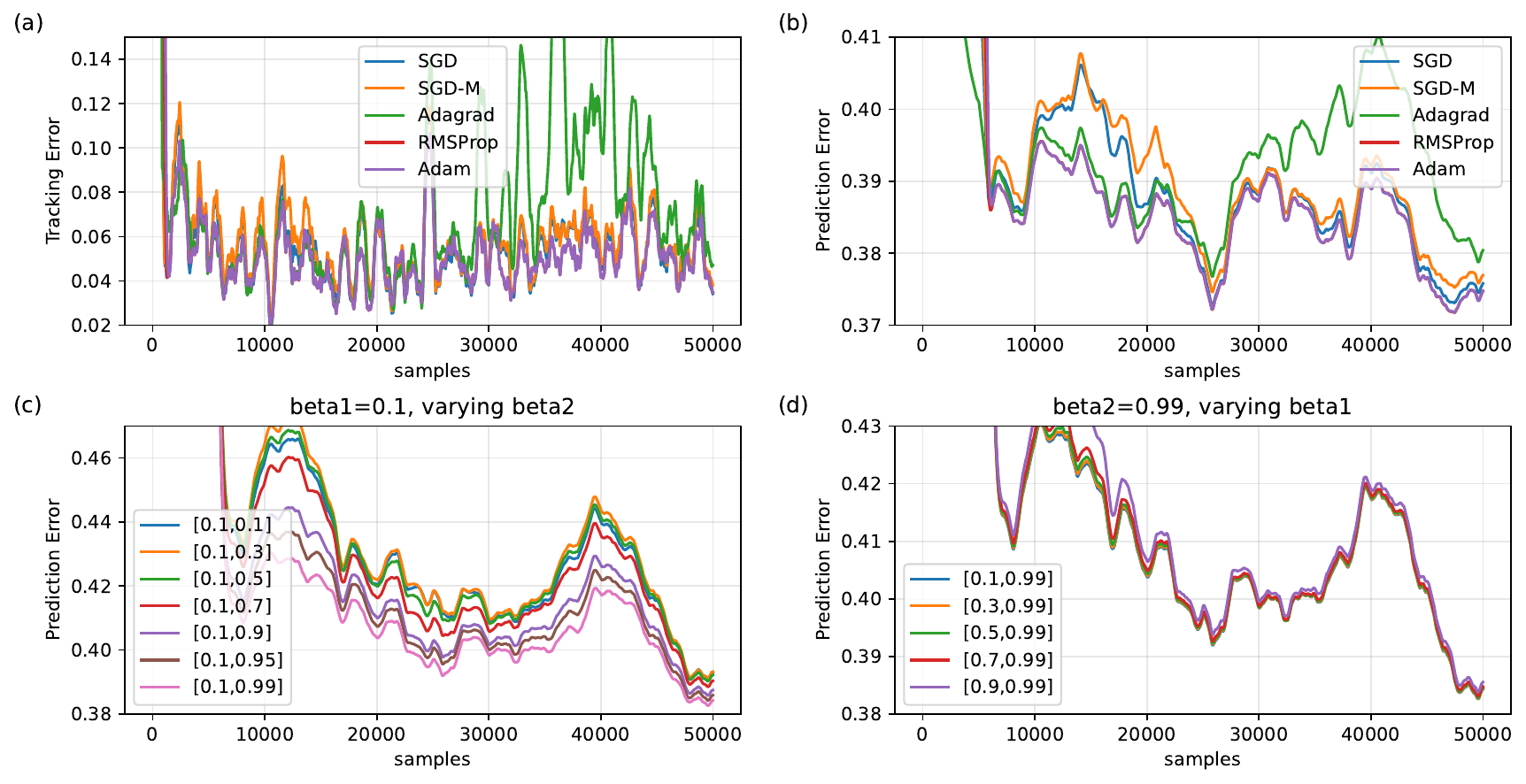}
\caption{Parameter tracking error and prediction error under different online optimization algorithms, and sensitivity of Adam to hyperparameters $(\beta_1,\beta_2)$.}
\label{fig:dtyf}
\end{center}
\end{figure}

We compare five standard online optimization methods (see, e.g., \cite{goodfellow2016dl}): \emph{SGD} with stepsize $\alpha=10^{-2}$, \emph{SGD with Momentum} ($\alpha=10^{-2}, \beta=0.9$), \emph{AdaGrad} ($\alpha=0.9, c_0=10^{-8}$), \emph{RMSProp} ($\alpha=10^{-2}, \beta=0.99, c_0 = 10^{-8}$), and \emph{Adam} ($\alpha=10^{-2}, \beta_1=0.1, \beta_2=0.99, c_0=10^{-8}$). For Adam, the parameter update is followed by a projection onto the admissible set $D=[-10,10]^5$, which guarantees bounded estimates. All methods perform updates based on the instantaneous squared loss, i.e., $(y_{k+1}-\hat{y}_{k+1})^2$, without mini-batches or offline training. Performance is evaluated using two metrics:  
(i) the parameter tracking error $\|\tilde{\theta}_k\|^2$, and  
(ii) a rolling-window average of the squared prediction loss over the most recent $W=5000$ samples. As shown in Fig.~\ref{fig:dtyf}(a), Adam achieves the lowest parameter tracking error throughout the time horizon, indicating effective tracking of the drifting parameters.
Figure~\ref{fig:dtyf}(b) shows that Adam also achieves the lowest rolling mean prediction error compared to the other methods. 
In contrast, AdaGrad exhibits larger fluctuations and higher regret, indicating its sensitivity to nonstationarity. It should be noted that the hyperparameter choice used for Adam, namely the step size $\alpha=\sqrt{C_\Delta}=10^{-2}$ together with a small $\beta_1$ and a large $\beta_2$, is consistent with the intuitive guidelines in Remark~\ref{regret_intuitive} and aligns with our theoretical results.

We further investigate the sensitivity of Adam to its hyperparameters $(\beta_1,\beta_2)$ from a prediction-error perspective.
Fixing $\beta_1=0.1$ and varying $\beta_2 \in {(0.1,0.3,0.5,0.7,0.9,0.95,0.99)}$, Fig.~\ref{fig:dtyf}(c) exhibits a clear performance stratification: larger $\beta_2$ consistently leads to lower prediction error and noticeably smoother trajectories.
In contrast, when $\beta_2$ is fixed at $0.99$, varying $\beta_1$ over ${(0.1,0.3,0.5,0.7,0.9)}$ yields curves that are largely close to one another in Fig.~\ref{fig:dtyf}(d), suggesting a weaker dependence on $\beta_1$ over a broad range. However, taking $\beta_1$ large (e.g., $\beta_1=0.9$) degrades prediction accuracy in time-varying and nonstationary regimes. Overall, these empirical observations align with our theoretical results in time-varying scenarios with nonstationary data and offer practical guidance for hyperparameter selection: choosing a sufficiently large $\beta_2$ together with a comparatively smaller $\beta_1$ yields lower prediction error.

\noindent
\textbf{Real-World Data Experiments.}

Having verified the performance of Adam under a synthetic drifting-parameter model, we next evaluate the practical impact of Adam’s momentum parameters $(\beta_1,\beta_2)$ in a real-world nonstationary time-series prediction task. Experiments are conducted on the UCI Air Quality dataset~\cite{DeVito2008AirQuality}, which exhibits pronounced temporal variability, where the goal is one-step-ahead prediction of the carbon monoxide (CO) concentration using an online linear model.
At each time step, the input feature vector includes a bias term, the past $12$ lagged values of the target variable, sinusoidal hour-of-day features capturing daily periodicity, and contemporaneous measurements from other sensors.
All features are standardized prior to training.

Model parameters are updated online via the Adam algorithm with stepsize $\alpha=10^{-3}$ and $c_0=10^{-8}$.
We perform a grid search over
$\beta_1 \in \{0, 0.02, 0.05, 0.1, 0.2, 0.3, 0.5, 0.7\}$ and
$\beta_2 \in \{0.9, 0.95, 0.99, 0.995, 0.999, 0.9995\}$.
After each update, the parameter vector $\theta_k \in \mathbb{R}^d$ is projected onto a bounded admissible set $D=[-10,10]^d$.
Performance is evaluated using the squared prediction loss $(y_{k+1}-\hat{y}_{k+1})^2$. The prediction error is defined as the average loss, i.e., $\frac{1}{n}\sum_{k=1}^n (y_{k+1}-\hat{y}_{k+1})^2$, without mini-batching or offline recalibration.
Figure~\ref{fig:beta_heatmap} shows that prediction performance exhibits a clear dependence on the second-moment parameter $\beta_2$. The left heat map shows a clear monotone dependence on $\beta_2$ across all $\beta_1$. Larger $\beta_2$ consistently yields lower average prediction error, and $\beta_2 \ge 0.999$ performs best. By comparison, sensitivity to $\beta_1$ is limited over a moderate range, but degrades for large $\beta_1$.

\begin{figure}[htbp]
\begin{center}
\includegraphics[width=1\linewidth]{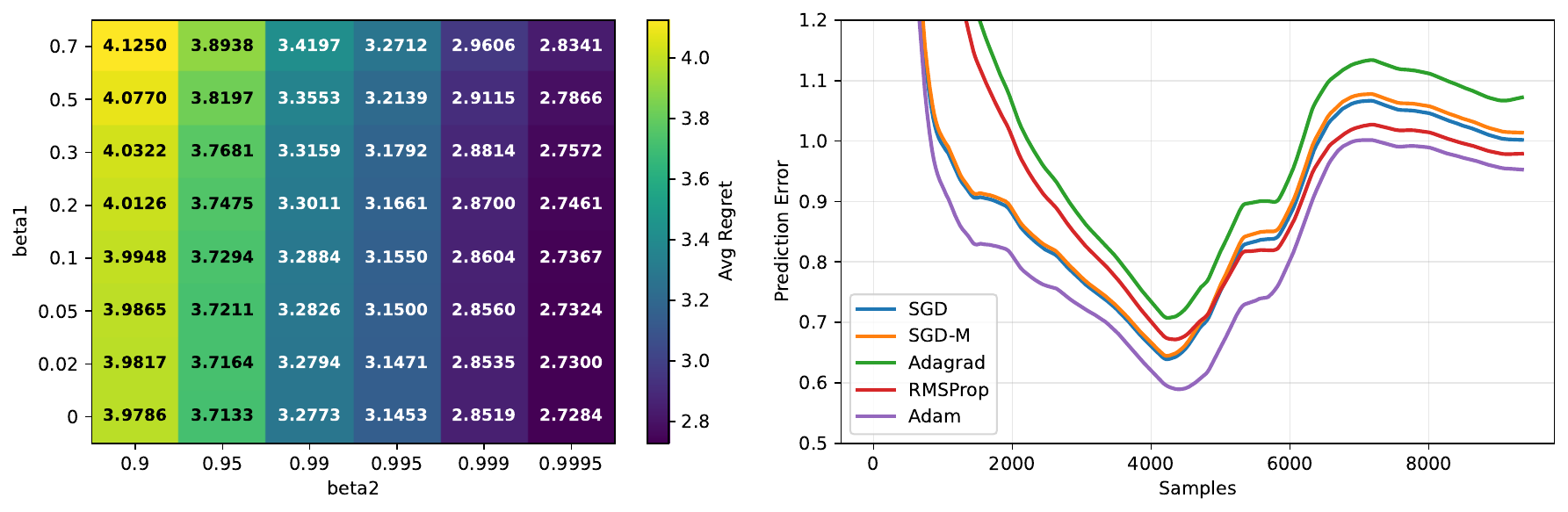}
\caption{Left: heatmap of average prediction error as a function of $(\beta_1,\beta_2)$. 
Right: average prediction error trajectories under different online optimization algorithms.}
\label{fig:beta_heatmap}
\end{center}
\end{figure}

The right panel compares Adam with four standard online optimization methods using fixed parameter configurations. Specifically, we consider: \emph{SGD} with stepsize $\alpha = 5 \times 10^{-3}$; \emph{SGD with momentum} $(\alpha = 5 \times 10^{-3}, \beta = 0.1)$; \emph{AdaGrad} $(\alpha = 0.5, c_0 = 10^{-8})$; \emph{RMSProp} $(\alpha = 10^{-2}, \beta = 0.9995, c_0 = 10^{-8})$; \emph{Adam} $(\alpha = 2 \times 10^{-2}, \beta_1 = 0.01, \beta_2 = 0.9995, c_0 = 10^{-8})$. For Adam, each parameter update is followed by a projection onto the admissible set $D=[-10,10]^d$, where $d$ denotes the feature dimension, ensuring bounded parameter estimates during the online learning process. While previous studies have demonstrated the advantages of Adam in the stationary settings, our experiment focuses on its performance under online nonstationary data streams. The results show that Adam produces significantly smoother and more stable prediction-error trajectories in this dynamic environment, highlighting its adaptation to distributional changes and time-varying data characteristics with suitable hyperparameters.

\section{Conclusion}\label{conclusion}

Despite the well-known Adam’s widespread practical success, providing rigorous guarantees for both parameter estimation and adaptive prediction remain challenging, particularly when data are generated by time-varying, nonstationary dynamical systems. By combining Lyapunov-based techniques with new tools for analyzing products of random matrices in Adam’s dynamics, we established explicit parameter-tracking and prediction guarantees for a broad class of nonlinear stochastic systems.
Our results rely on a conditional excitation condition on the data, which substantially relaxes the commonly imposed i.i.d. assumptions and therefore applies to feedback-driven and temporally dependent data streams. Moreover, compared with existing related  results, our bounds also feature a milder and better-conditioned dependence on the hyperparameters, yield sharper scalings with respect to momentum and second-moment adaptation, and explicitly quantify the combined impact of gradient noise and parameter drift. In this way, we provide theoretical guarantees that extend well beyond i.i.d. and constant parameter learning  settings.
These findings offer a theoretical foundation of Adam-type algorithm under general datasets and also suggest guidelines for hyperparameter selection, which are further supported by two experiments. Interesting directions for future work include extending the analysis to broader nonlinear model classes and developing systematic hyperparameter design principles.

\bibliographystyle{IEEEtran}
\bibliography{references}

\end{document}